\documentclass[10pt,twocolumn]{article} 
\usepackage{simpleConference}
\usepackage{times}
\usepackage{graphicx}
\usepackage{amssymb}
\usepackage{url,hyperref}

\usepackage[utf8]{inputenc}
\usepackage{amsmath}
\usepackage{algorithm}
\usepackage{algpseudocode}
\usepackage{booktabs}
\usepackage{graphicx}
\usepackage[strings]{underscore}
\DeclareMathSymbol{\mh}{\mathord}{operators}{`\-}
\usepackage[justification=centering]{caption}
\usepackage[english]{babel}
\usepackage{tabularx}
\usepackage[utf8]{inputenc}
\usepackage{amsthm}
\usepackage{amsmath}

\usepackage{graphicx}
\usepackage{float}
\newtheorem{observation}{Observation ---}
\newtheorem{theorem}{Theorem}

\usepackage[utf8]{inputenc}
\usepackage{amsmath}
\usepackage{adjustbox}
\usepackage{cite}

\begin{document}

\title{A Robust Optimization Method for Label Noisy Datasets Based on Adaptive Threshold: Adaptive-k}

\author{Enes Dedeoglu,  Himmet Toprak Kesgin, Mehmet Fatih Amasyali \\
  \textit{Department of Computer Engineering} \\
  \textit{Yildiz Technical University} \\
  Istanbul, Turkey\\
  \texttt{enes.dedeoglu@std.yildiz.edu.tr} \\
  \texttt{\{tkesgin, amasyali\}@yildiz.edu.tr}
}
\maketitle
\thispagestyle{empty}

\begin{abstract}
SGD does not produce robust results on datasets with label noise. Because the gradients calculated according to the losses of the noisy samples cause the optimization process to go in the wrong direction. In this paper, as an alternative to SGD, we recommend using samples with loss less than a threshold value determined during the optimization process, instead of using all samples in the mini-batch. Our proposed method, Adaptive-k, aims to exclude label noise samples from the optimization process and make the process robust. On noisy datasets, we found that using a threshold-based approach, such as Adaptive-k, produces better results than using all samples or a fixed number of low-loss samples in the mini-batch. Based on our theoretical analysis and experimental results, we show that the Adaptive-k method is closest to the performance of the oracle, in which noisy samples are entirely removed from the dataset. Adaptive-k is a simple but effective method. It does not require prior knowledge of the noise ratio of the dataset, does not require additional model training, and does not increase training time significantly. The code for Adaptive-k is available at https://github.com/enesdedeoglu-TR/Adaptive-k.
\end{abstract}

\section{Introduction}
Training deep learning models usually involves large datasets that are prone to label noise.
When the model is trained on a dataset that contains label noise, the classification performance decreases on the test dataset.
In addition, deep learning models can easily fit the datasets with random labels \cite{zhang2021understanding}. 

In the label noisy datasets, the noisy samples are separated at the beginning and the training made with only noiseless samples is more successful than the training with all samples.
However, in order to make this distinction in a real dataset, it is necessary to know which samples have noise, which may be very costly and in some cases not possible.
On the other hand, samples with label noise tend to have high loss values in the early stages of optimization \cite{liao2020weakly, ortego2021towards, arazo2019unsupervised,  nishi2021augmentation, majidi2021exponentiated}.
Based on this information, it has been shown that performing the updates only with the smallest k samples in loss, instead of all samples in the mini-batch, increases the success in label noisy datasets \cite{shah2020choosing}.
Yet, there is a difference in success between using only noiseless samples and using k samples with the smallest loss, which is Min-k Loss (MKL) in the mini-batch.
This means that the detection of noisy samples can be done even better.

In this study, in order to narrow this performance gap, it is proposed to adaptively determine the number of samples selected for updating from the mini-batch. Instead of choosing k samples for each mini-batch, we argue that choosing a variable number of samples according to the loss distributions of the previous and current samples in mini-batches will improve the optimization process. Because samples are randomly distributed to the mini-batch, samples with the highest loss in the mini-batch may not be labeled as noisy. For example, if all the samples are noiseless in the mini-batch, not using some of them will slow down the optimization process. However, if all the samples are noisy, it is necessary to eliminate all of them so that they do not mislead the optimization in the wrong direction. For these reasons, it is not possible to determine the adaptive k by only looking at the losses of the samples in the existing mini-batch. From this point of view, we propose a method in which the k sample selection threshold is determined by using the losses of previous mini-batches, inspired by adaptive optimization methods.

It has been shown that the use of adaptive k instead of fixed k in experiments better distinguishes samples with label noise. In other words, the proposed adaptive k method has been shown to increase success in label noisy datasets. Since samples with label noise will have higher loss values than samples without noise, not using samples with high loss values in the mini-batch prevents the optimization from slowing down and/or going in the wrong direction. The results showed that the Adaptive-k method is closest to the success of the oracle method, in which noisy samples are deleted from the dataset and not used in the optimization process. 
Adaptive-k does not necessarily have to be used with label noisy datasets. In the absence of label noise, it can perform (almost) as well as SGD (vanilla) by excluding very few samples during training. It is possible to associate this with curriculum learning, which reduces the priority of difficult samples and in some conditions, may improve vanilla performance \cite{bengio2009curriculum, kesgin2022cyclical}.

The primary contributions of this research are as follows: 

\begin{itemize}
\item Introducing Adaptive-k, a novel algorithm for robust training of label noisy datasets.
\item Theoretical analysis of Adaptive-k and MKL algorithm comparison with SGD.
\item High accuracy noise ratio estimation using the proposed Adaptive-k without making hyperparameter adjustments or having prior knowledge of the dataset.
\item Along with the theoretical analysis, In empirical research, Oracle, SGD, MKL, SGD-MKL, and Adaptive-k algorithms are compared on three images and four text datasets.
\item Additionally, the Adaptive-k algorithm is easy to implement, works in a mini-batch level, and doesn't require training additional models\cite{han2018co} or data augmentation \cite{yang2021learning}.
\end{itemize}

\section{Related Work}
In machine learning models trained with SGD, the presence of outliers and noisy values can significantly skew the model parameters. MKL, which was developed to avoid this effect, is a simple variant of the SGD method \cite{shah2020choosing}. At each step, a set of samples is selected for the mini-batch, then an SGD-like update is performed on only the samples with the smallest k losses. Algorithm-1 shows the pseudocode of the MKL algorithm. Because the k number is fixed in the MKL method, always k samples are selected in each update. Previous mini-batch updates and sample distributions in the current mini-batch are not considered.


\begin{algorithm}[h]
\caption{MKL}
\begin{algorithmic}[1]
\Require
\State $T$ : Epoch count.
\State $N$ : Iteration count.
\State $D$ : Training dataset.
\State $f$ : Loss function.
\State $k$ : Number of samples with the lowest loss chosen

\Statex

\Procedure {train}{}
\For {$i \in \{1...T\}$}
\For {$j \in \{1...N\}$}
\State Fetch mini-batch $\overline{D}$ from $D$
\State Sort $f(\overline{D})$ in ascending order
\State $\overline{D'} \leftarrow \overline{D}[1,...,k]$
\State Update model with $\overline{D'}$

\EndFor
\EndFor

\EndProcedure
\end{algorithmic}
\end{algorithm}

Another study proposed using a function called R(t) to determine the number of samples to be chosen in the mini-batch \cite{yao2020searching}. The function is determined not based on selected samples, but based on time. The Optimal R(t) function should be determined for each dataset. Additionally, learning the R function takes a long time. Our method makes a mini-batch-specific selection that isn't specific to a dataset, and instead, uses the moving average instead of time. 

MedianTGD is another method proposed for improving the robustness of gradient descent against outliers \cite{chi2019median}. They suggested using the median metric to separate outlier samples, since the median metric is effective in detecting outliers. At each iteration, the median-TGD eliminates the contributions of samples that differ significantly from the sample median to stabilize the direction of the search while performing the update. Authors present numerical experiments that show median-TGD to have superior performance.

The co-teaching method was proposed to combat noisy labels, which simultaneously trains two different models\cite{han2018co}. The models  communicate with each other during training about which samples to trained for each mini-batch. The aim is to remove noisy samples in this way. The robustness of their method was demonstrated on noisy image datasets.

Trimmed Loss Minimization is another study that only includes samples with the smallest loss in training \cite{shen2019learning}. In this method, the smallest losses are selected on an epoch basis, not on a mini-batch basis. They showed that in the early stages of training, the models were more accurate on clean samples. In their experiments, they demonstrated the effectiveness of their method for deep image classifiers with label noise, generative adversarial networks with bad training images, and backdoor attacks. 

Label noise can also be used to regularize deep neural networks\cite{nakamura2020regularization}. Before the mini-batch update, a certain rate of random label noise is added to the samples. And samples with current loss above and below a certain threshold are trimmed, and back propagation is performed. Label-Noised Trim-SGD yielded more robust results for some datasets and models.

Some papers recommend the use of augmentation methods for labeling noisy datasets \cite{arazo2019unsupervised, yang2021learning}.
It is attempted to predict noisy samples more accurately by using augmented data.
However, these studies and methods are limited to image datasets since text augmentation do not yield satisfactory results as on images.

In contrast, Adaptive-k algorithm does not use any augmentation methods and is not specific to image datasets and can be used for any kind of dataset. Additionally, Adaptive-k does not require any prior information about the noise ratio. Therefore, there is no need to tune the hyperparameters. Adaptive-k does not involve heavy calculations; the operations it contains are negligible when compared to mini-batch updates. In addition, it is not necessary to use a pre-trained model or train another model in parallel, as with co-teaching methods, in the Adaptive-k algorithm. As a result, it has a negligible impact on the cost of training.

\section{Method}

Our method consists of two steps: the vanilla process and the adaptive process. The first step, the vanilla process, includes a training process in which all samples participate in the training. The vanilla process is followed by the adaptive phase. In the adaptive process, after the samples are selected with an adaptive threshold value, training is performed with only the selected samples. 

We observed that not using vanilla method at the beginning of the training differed the loss distributions of noisy and clean samples.
we therefore recommend the vanilla process before an adaptive approach in the early stages of training.
Similarly, when we apply the vanilla process before MKL, we see that the results are superior to the standard MKL. As a result, another approach we recommend is to use vanilla and MKL together.

The vanilla process aims to reduce sample losses to a significant level, that is, to the point where the losses of noisy samples are generally high and the losses of clean samples are generally low. 
Following the vanilla process, we begin to distinguish between samples using an adaptive threshold. This threshold value is determined using the moving average calculated in each mini-batch.

Having the gradient of all samples would make it easy to distinguish between noisy and clean samples. However, considering the large datasets of today, it is not possible to calculate gradients of all samples. Therefore, we make this problem computationally tractable by utilizing the moving average to extract a general characteristic of all samples. If we only looked at the mean of the incoming mini-batch, this metric would have misled us. Making a selection using only the mean of the incoming mini-batch will result in undesirable situations, such as not receiving some clean samples in a mini-batch with a majority of clean samples, or some noisy samples participating in the training process in a mini-batch with a majority of noisy samples. Therefore, we argue that it is a more appropriate metric to use the moving average calculated using the mean of the mini-batches to make an accurate distinction. We update the threshold value using the mean of each mini-batch by using a moving average, and in this way we try to catch the appropriate threshold. By using the threshold, we make a choice for which samples to use for an update. And we perform the update process by including the selected samples in the training process. We were inspired by the Adam optimization algorithm\cite{kingma2014adam} to calculate the moving average. The Adaptive-k algorithm differs from the MKL algorithm in the way it determines the number of k. A fixed value of k is initially set in MKL, and k samples are used within each mini-batch. Conversely, Adaptive-k uses a different number of samples from each mini-batch based on the threshold it determines. Training algorithm of Adaptive-k is given in Algorithm-2 and training algorithm of MKL is given in Algorithm-1.

\begin{algorithm}[h]
\caption{Adaptive-k}
\begin{algorithmic}[1]
\Require
\State $T$ : Epoch count.
\State $N$ : Iteration count.
\State $D$ : Training dataset.
\State $f$ : Loss function.
\State $ \beta_1, \beta_2 \in ({0,1}] $ : Exponential decay rates for the m and the v estimates.(Suggested default: 0.9 and 0.999 respectively)
\State $\varepsilon$ : Small constant used for numerical stabilization.(Suggested default: 10e-8)

\Statex

\Procedure {train}{}
\State $m_0\leftarrow$  0
\State $v_0\leftarrow$  0
\For {$i \in \{1...T\}$}
\For {$j \in \{1...N\}$}
\State Fetch mini-batch $\overline{D}$ from $D$
\State $losses\leftarrow$ $f(\overline{D})$ 
\State  $\mu_b\leftarrow$ \Call{Mean}{$losses$} (Calculating the average loss for the mini-batch)
\State $k = ((i-1)*T+j)$
\State $m_{(k)}\leftarrow$ $\beta1 \times m_{(k-1)} + (1-\beta_1) \times \mu_b$
\State $v_{(k)}\leftarrow$ $\beta2 \times v_{(k-1)} + (1-\beta_2) \times {\mu_b^2}$
\State $\mu_D \leftarrow m_{(k)} / $ $(\sqrt{v_{((k)}} +\varepsilon)$ (Estimating the average
loss for all samples)
\State $\overline{D'} \leftarrow \overline{D}[\overline{D} \leqslant \mu_D]$ (Selecting samples below threshold)
\State Update model with $\overline{D'}$

\EndFor
\EndFor

\EndProcedure
\end{algorithmic}
\end{algorithm}

\subsection{Understanding Adaptive-k}

We argue that the MKL method developed for a dataset containing label noisy data can be improved. The number k, which indicates the number of samples taken from the mini-batch in the method suggested by Vatsal Shah et al.\cite{shah2020choosing}, is an important parameter that must be determined well and appropriately for this method to yield the desired results. In the MKL method, the number k is used as a fixed number. In this case, we will need to determine this k parameter according to the noise ratio in the dataset. However, since we cannot predict the noise ratio in the datasets, we can only find the number k by testing it. This is a laborious process and may not also be a definitive solution. 

As tested in \cite{shah2020choosing}, when the k number is determined based on the noise ratio, this is an optimum solution for MKL, but it is far from being an optimum solution, in general. We can see the precision and recall shapes that occur when we perform these tests, in order to show that this situation is not an optimum solution, in Figure-\ref{fig:u_1_prec_recall_for_MKL}. Precision is defined as the fraction of clean samples among the selected samples. Recall, on the other hand, is defined as the fraction of clean samples that were selected.

\begin{figure}[h]
\begin{center}
\includegraphics[width=\columnwidth]{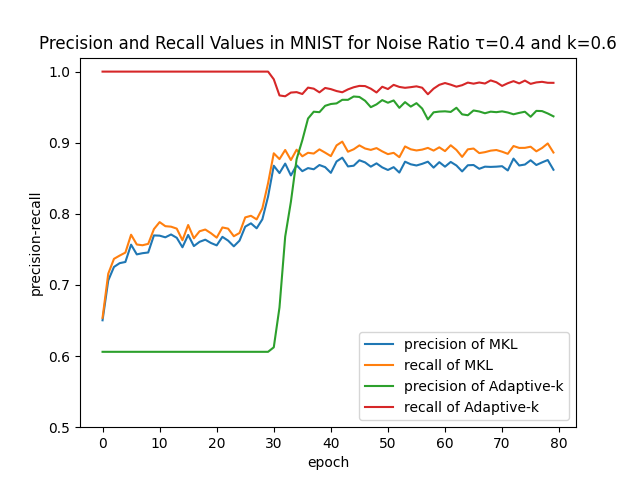}
\caption{Precision and Recall Values in MNIST for Noise Ratio $\tau = 0.4$ and $k=0.6$. For the MNIST dataset, after the first 30 epochs(Vanilla), the Adaptive-k process, which is 50 epochs, is started. Therefore, the first 30 epochs of Adaptive-k show the precision and recall values in the vanilla process.}
\label{fig:u_1_prec_recall_for_MKL}
\end{center}
\end{figure}

For label noise datasets, our goal should be to convergence precision and recall values closer to one if our purpose is to distinguish clean data from noisy data. Figure-\ref{fig:u_1_prec_recall_for_MKL} shows the average precision and recall values calculated in each iteration in an epoch. When we look at Figure-\ref{fig:u_1_prec_recall_for_MKL}, precision and recall values are far from 1 even in the most optimal case for MKL in the MNIST dataset. 

As mentioned earlier, we cannot foreknow how much label noise a dataset contains in real life. Even in the case where we assume that we foreknow the label noise ratio and determine the k number accordingly, we see in Figure-\ref{fig:u_1_prec_recall_for_MKL} that the noisy-noiseless distinction could not be made clearly with MKL. In simple terms, if a\% of the dataset is label noisy, it seems like a good solution to determine k so that (100-a)\% is chosen at first glance. However, this situation is may not as simple as it seems. a\% of a dataset is label noisy does not mean that a\% of mini-batches will also be label noisy. The same situation can also be seen in Figure-\ref{fig:u_2_rates}. Figure-\ref{fig:u_2_rates} shows the clean sample rate in a mini-batch and the sample rates selected by different approaches from that mini-batch. The clean sample rate is different in each mini-batch. If this were the situation, MKL would have been a definitive solution, assuming that we know that the losses of the noisy data are higher than the noiseless data, and we also know the noise ratio in the dataset. Because of this situation, even determining the k number according to the noise ratio does not give us the desired optimum solution. As we can see in Figure-\ref{fig:u_2_rates}, a fixed k number will cause situations that we do not want to happen, such as some noisy samples participating in the training process if the majority of the incoming mini-batch is noisy, or some noiseless samples not participating in the training process in the mini-batch consisting of mostly noiseless samples. While MKL selects a fixed number of k samples, our proposed Adaptive-k method chooses according to the distribution of the mini-batch in that iteration. We can see this situation in Figure-\ref{fig:u_2_rates}.

\begin{figure}[h]
\begin{center}
\includegraphics[width=\columnwidth]{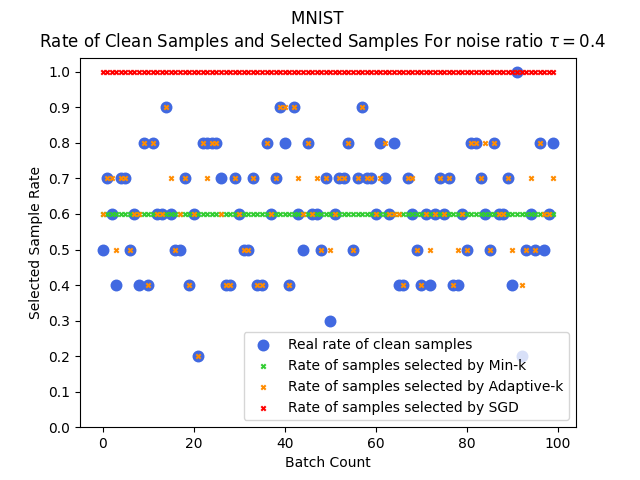}
\caption{Rate of Clean Samples and Selected Samples in MNIST for Noise Ratio $\tau = 0.4$. This figure shows the iterations in any epoch during the training phase.}
\label{fig:u_2_rates}
\end{center}
\end{figure}

The study\cite{shah2020choosing} also shows that it is possible to get even more successful results by perfectly separating the noiseless data from the noisy data. A perfect distinction between clean and noisy samples is referred to as the Oracle method.
Our motivation, due to the fact that we do not know in advance how much label noise the datasets contain and, as we mentioned earlier, some negative aspects of MKL, we think that adaptively determining the number of k in each mini-batch will be an important step in the development of the current MKL method and in this way we will approach the Oracle level. Therefore, our goal with this study is to reach the Oracle level by making a near-perfect distinction. We argue that instead of taking a fixed number of samples in the mini-batch during the training process, choosing an adaptive threshold value will yield better results. In this way, we believe that we will approach even more the Oracle level.

Thanks to our proposed method, an adaptive threshold value will be determined, and a selection will be made among the samples in the mini-batch according to this threshold value. Thus, deciding a fixed number of samples from the mini-batch will be prevented, and a selection will be made according to the current mini-batch status.

\section{Experiments}
In order to test our proposed method, we performed tests with datasets consisting of image and text data. While we used datasets that are frequently used in deep learning such as subsampled MNIST\cite{deng2012mnist}(5000 training samples, 10000 test samples), subsampled Fashion-MNIST\cite{xiao2017fashion}(5000 training samples, 10000 test samples) and CIFAR-10\cite{krizhevsky2009learning}(50000 training samples, 10000 test samples) as image datasets, we used datasets from IMDB\cite{imdb}(25000 training samples, 25000 test samples), Sarcasm\cite{sarcasm, book}(25757 training samples, 2862 test samples), Hotel\cite{hotel}(13128 training samples, 6763 test samples) and 20newsgroup\cite{20_news}(15998 training samples, 3999 test samples) as text datasets. We randomly added label noise via directed label noise to the datasets that we used in our tests.

The results are the average of the maximum test accuracies obtained during the training over 3 runs. The results are given in the Table-\ref{tbl:table}. Taking into account different noise ratios, we ran tests for 28 different situations in 7 different datasets. Adaptive-k obtained higher test accuracy from vanilla in 21, and MKL in 27 of 28 tests. In 18 results, Adaptive-k was closest to Oracle in terms of test accuracy.

\begin{figure}[h]
\begin{center}
\includegraphics[width=\columnwidth]{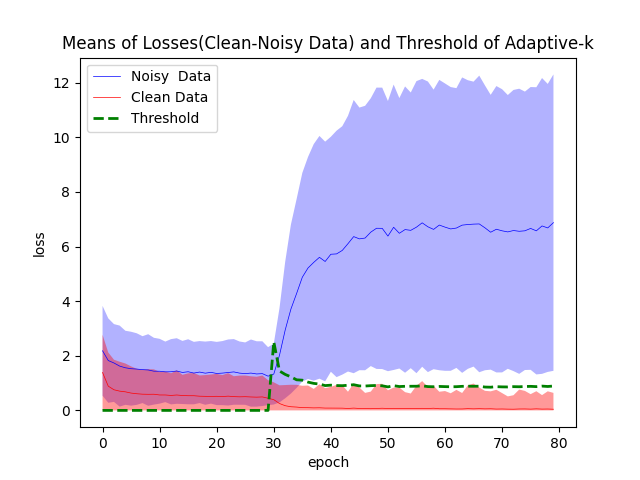}
\caption{The means of losses of clean and noisy samples in MNIST dataset for $\tau$=0.3 and threshold value determined by adaptive k for separation of these samples. The blue line shows the mean losses of noisy samples, the red line shows the mean losses of clean samples, and the green dashed line shows the calculated and used threshold value ($\mu_d$ in Adaptive-k Algorithm 2) threshold value. The shaded area shows the ranges as mean $\pm$ 1.5*standart deviation. In the MNIST dataset, after the 30th epoch, the Adaptive-k process is started.}
\label{fig:thres.png}
\end{center}
\end{figure}

\begin{table}[H]
\begin{center}
\includegraphics[width=\columnwidth]{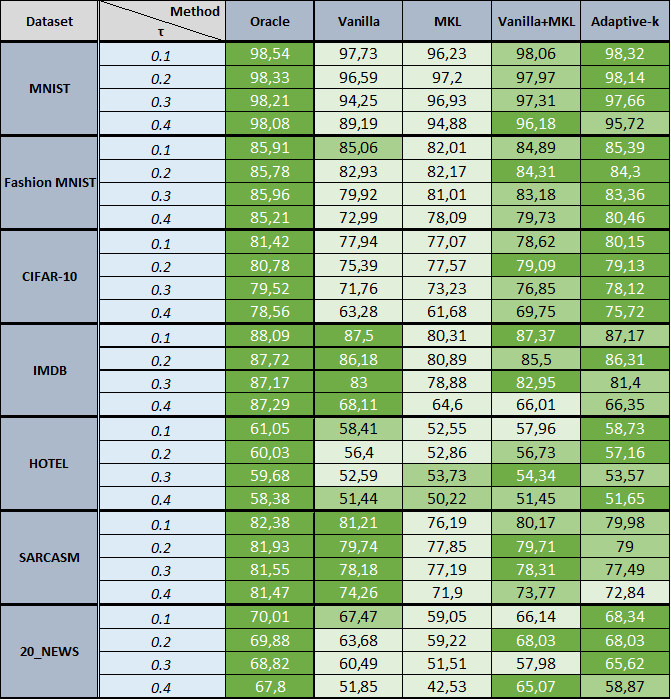}
\caption{Comparing of Adaptive-k and Vanilla+MKL that we propose in different datasets with Oracle, Vanilla and MKL on test accuracy. Shades of green colour indicate closeness to the Oracle level we are aiming(e.g. darker green means closer).}
\label{tbl:table}
\end{center}
\end{table}

\begin{figure}[h]
\begin{center}
\includegraphics[width=\columnwidth]{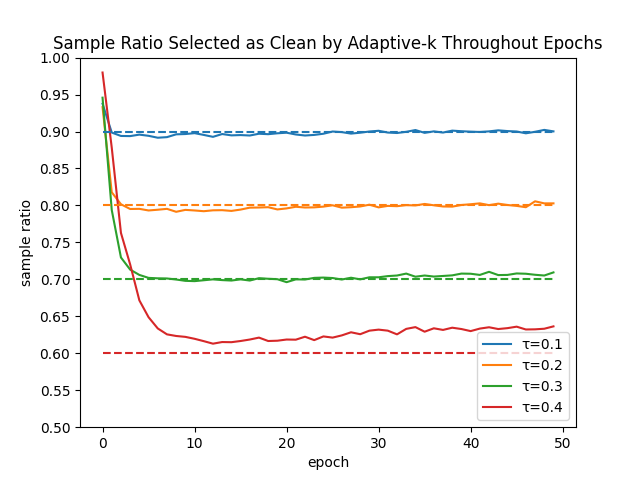}
\caption{shows only the adaptive process for the MNIST dataset. For the MNIST dataset, after the first 30 epochs(Vanilla), the Adaptive-k process, which is 50 epochs, is started. Solid lines indicate sample ratios determined as clean by Adaptive-k, while dashed lines indicate actual clean sample ratios.}
\label{fig:avg_k.png}
\end{center}
\end{figure}

Figure-\ref{fig:thres.png} shows threshold value of Adaptive-k during training. Considering that the losses of the unselected samples will increase, we can observe that we make a good distinction in Figure-\ref{fig:thres.png}.

To compare our proposed method with MKL, it would be meaningful to look at the precision and recall values of these approaches, shown in Figure-\ref{fig:u_1_prec_recall_for_MKL}. Adaptive-k's precision and recall values are higher than MKL. Naturally, this situation is also reflected in the results. Even in the most optimal case($\tau$=0.4, k=0.6) of MKL, our Adaptive-k method can give better results.

Figure-\ref{fig:avg_k.png} shows the average of the k-numbers determined for the mini-batches in each epoch during the adaptive training process. Average of k corresponds to clean sample ratio estimation of Adaptive-k. With the Adaptive-k approach, we argue that k converges to the ideal point for the dataset. In addition, when we look at the average k numbers determined by Adaptive-k at controlled noise ratios for the MNIST dataset, we estimate how much a dataset contains label noise with high precision.

\section{Theoretical Analysis}

If deep learning models are trained on datasets containing label noise, then label noisy samples tend to have higher loss. In the absence of noisy samples, the classification performance of the model increases. Our threshold-based Adaptive-k algorithm seeks to avoid training with noisy samples and to use only noiseless samples during training. A comparison is made between the proposed algorithm and the Min-k Loss SGD algorithm, which selects a fixed number of samples with low loss in mini-batch, and the SGD (vanilla) algorithm, which selects all samples without making any distinction. In this analysis, we will call  Min-k Loss SGD algorithm as MKL, with k as the number of samples selected from a mini-batch. For the theoretical analysis, it is assumed that the noiseless and noisy samples come from two different normal distributions, with the noiseless samples having a lower mean. Samples in the dataset are come from a mixture of these two distributions. The mean square error (MSE) of the Adaptive-k, MKL, and vanilla methods was calculated based on the noiseless sample's distribution. And it has been demonstrated in which cases Adaptive-k has less error than other methods.

\subsection{Notations and Definitions}

\begin{table}[H]
\begin{adjustbox}{width=0.75\columnwidth,center}
\begin{tabularx}{\linewidth}{|l|X|}
\hline
\textbf{Notation} & \textbf{Description}                                         \\ \hline
$X_1$             & Clean Samples                                                \\ \hline
$X_2$             & Noisy Samples                                                \\ \hline
$\mu_1$           & Mean loss of clean samples                                   \\ \hline
$\mu_2$           & Mean loss of noisy samples                                   \\ \hline
$\sigma_1$        & Standard deviation of losses for clean samples               \\ \hline
$\sigma_2$        & Standard deviation of losses for noisy samples               \\ \hline
$\mathcal{N}$     & Normal Distribution                                          \\ \hline
$f_1(x)$          & Probability density function of losses for clean samples     \\ \hline
$f_2(x)$          & Probability density function of losses for noisy samples     \\ \hline
$F_1(x)$          & Cumulative distribution function of losses for clean samples \\ \hline
$F_2(x)$          & Cumulative distribution function of losses for noisy samples \\ \hline
$f_D(x)$          & Probability density function of losses for all samples       \\ \hline
$F_D(x)$          & Cumulative distribution function of losses for all samples   \\ \hline
$\mu_D$           & Mean loss of all samples                                     \\ \hline
$\sigma_D$        & Standard deviation of losses for all samples                 \\ \hline
$f_{adk}(x)$      & Probability density function of Adaptive-k distribution      \\ \hline
$\mu_{adk}$       & Mean of Adaptive-k distribution                              \\ \hline
$\sigma_{adk}$    & Standard deviation of Adaptive-k distribution                \\ \hline
$MSE_{adk}$       & Mean squared error of Adaptive-k distribution                \\ \hline
$f_{MKL}(x)$    & Probability density function of MKL distribution           \\ \hline
$\mu_{MKL}$     & Mean of MKL distribution                                   \\ \hline
$\sigma_{MKL}$  & Standard deviation of MKL distribution                     \\ \hline
$MSE_{MKL}$     & Mean squared error of MKL distribution                     \\ \hline
\end{tabularx}
\end{adjustbox}
\end{table}

Let $X_1$ represent noiseless samples and $X_2$ represent noisy samples.
$X_1$ is normally distributed with mean $\mu_1$ and standard deviation $\sigma_1$
$$X_1 \sim \mathcal{N}(\mu_1,\,\sigma_1^{2})$$
and 
$X_2$ is normally distributed with mean $\mu_2$ and standard deviation $\sigma_2$
$$X_2 \sim \mathcal{N}(\mu_2,\,\sigma_2^{2})$$

The pdf functions of $X_1$ and $X_2$ can be written as follows.

$$f_1(x) = \frac{1}{\sigma_1\sqrt{2\pi}} 
  \exp\left( -\frac{1}{2}\left(\frac{x-\mu_1}{\sigma_1}\right)^{\!2}\,\right)
$$

$$f_2(x) = \frac{1}{\sigma_2\sqrt{2\pi}} 
  \exp\left( -\frac{1}{2}\left(\frac{x-\mu_2}{\sigma_2}\right)^{\!2}\,\right)
$$

Similarly, CDF functions can be written as follows.

$$F_1(x) = \frac{1}{2}\left[1 + erf{(\frac{x - \mu_1}{\sigma_1 \sqrt{2}})}\right]
$$

$$F_2(x) = \frac{1}{2}\left[1 + erf{(\frac{x - \mu_2}{\sigma_2 \sqrt{2}})}\right]
$$

When the noise ratio is defined as $\tau$ and dataset is defined as $D$, the pdf, cdf, mean, and variance of the distribution of the dataset samples are shown in Equation-1-4
\begin{flalign}
f_D(x) = (1 - \tau).f_1(x) + \tau. f_2(x)
\end{flalign}
\begin{flalign}
F_D(x) = (1 - \tau).F_1(x) + \tau. F_2(x)
\end{flalign}
\begin{flalign}
\mu_D = (1 - \tau).\mu_1(x) + \tau. \mu_2(x)
\end{flalign}
\begin{flalign}
\sigma^2_D &= (1 - \tau) \sigma_1^2+\tau \sigma_\tau+(1 - \tau) (\mu_1 - \mu_D)^2 \\
& + \tau(\mu_2-\mu_D)^2 \nonumber
\end{flalign}

An example of the mixture distribution for $\mu_1 = 0$, $\sigma_1 = 1$, $\mu_2 = 5$, $\sigma_2 = 2$, $\tau = 0.4$   can be seen in Figure-\ref{fig:Figure_1.png}.

\begin{figure}[H]
\begin{center}
\includegraphics[width=\columnwidth]{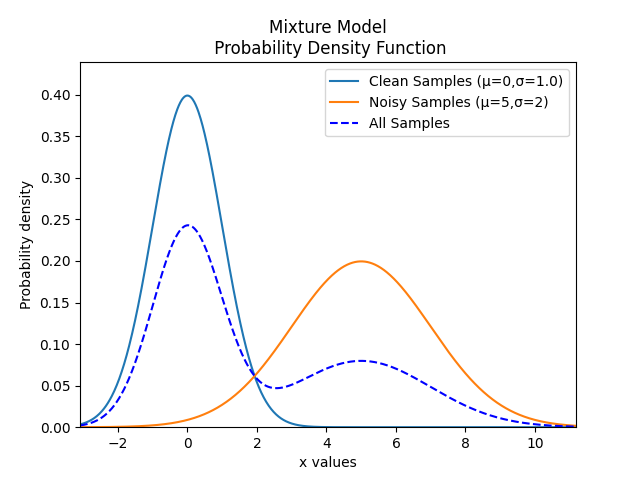}
\caption{Mixture distribution for noisy dataset where $\mu_1 = 0$, $\sigma_1 = 1$, $\mu_2 = 5$, $\sigma_2 = 2$, $\tau = 0.4$ }
\label{fig:Figure_1.png}
\end{center}
\end{figure}

\begin{observation}
As deep learning models are trained on label noisy datasets, label noisy samples tend to have higher loss. As a consequence, the mean loss of the label noisy samples is higher than clean samples.
\end{observation}

\begin{observation}
Model's classification performance increases in the absence of label noisy samples.
\end{observation}

\begin{theorem}
According to clean sample's distribution, the expected error of applying MKL algorithm (while updating selects the smallest k samples in mini-batch) is smaller than SGD.
\begin{equation}
    MSE_{MKL} < MSE_{SGD}
\end{equation}
\end{theorem}

\begin{proof}
MSE can be calculated as the sum of the variance and squared bias of the estimator. SGD's mean is defined as the mixture distribution created by the noisy and clean sample distributions. Therefore, the bias of the SGD according to the clean sample distribution can be calculated as follows.
\begin{equation}
    Bias(SGD) = (\mu_D - \mu_1)^2
\end{equation}
Equation-4 gives the variance of the SGD. Accordingly, the MSE of SGD is as follows.

\begin{flalign}
\sigma^2_D &= (1 - \tau) \sigma_1^2+\tau \sigma_\tau+(1 - \tau) (\mu_1 - \mu_D)^2 \\
& + \tau(\mu_2-\mu_D)^2 \nonumber  \\ 
& Var(SGD) = \sigma^2_D \nonumber 
\end{flalign}

\begin{flalign}
MSE_{SGD} = (\mu_D - \mu_1)^2 + \sigma^2_D
\end{flalign}

The MKL algorithm selects k samples with the smallest loss from a mini-batch with n samples. By using the order statistics pdf, it can be constructed the pdf of the MKL algorithm. A sample's $k^{th}$ order statistic is defined as its $k^{th}$ smallest value. $k^{th}$ order statistic's pdf is written as follows where $f_k(x)$ is pdf of $k^{th}$ order statistic, n is the number of observations in a mini-batch.
\begin{equation}
f_k(x) = n.f_D(x).{n - 1 \choose k - 1}. F_D(x)^{k-1}.(1-F_D(x))^{n-k}
\end{equation}
The pdf of MKL is simply the arithmetic mean of the order statistic from 1 to k, since the MKL algorithm only uses the smallest k samples instead of the $k^{th}$ the smallest sample and calculated as follows.
\begin{flalign}
f_{MKL}(x) &= \frac{1}{k} \sum_{p=1}^{k} n.f_D(x).{n - 1 \choose p - 1}. \\
& F_D(x)^{p-1}.(1-F_D(x))^{n-p} \nonumber
\end{flalign}

An example of the MKL distribution for $\mu_1 = 0$, $\sigma_1 = 1$, $\mu_2 = 5$, $\sigma_2 = 2$, $\tau = 0.4$ can be seen in Figure-\ref{ref:Figure_3.png}.

\begin{figure}[H]
\begin{center}
\includegraphics[width=\columnwidth]{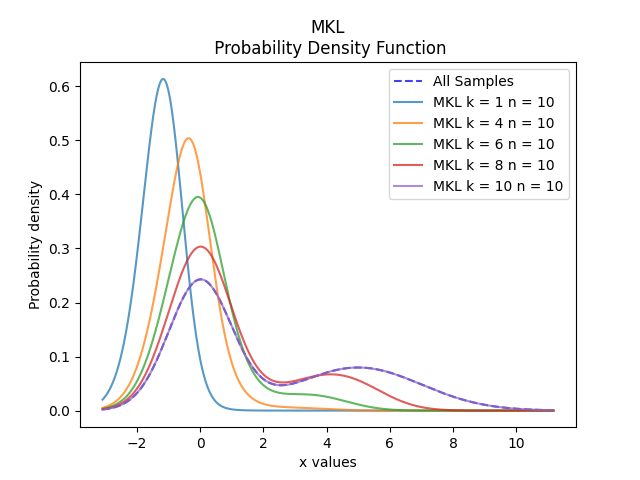}
\caption{MKL distribution's pdf for noisy dataset where $\mu_1 = 0$, $\sigma_1 = 1$, $\mu_2 = 5$, $\sigma_2 = 2$, $\tau = 0.4$ }
\label{ref:Figure_3.png}
\end{center}
\end{figure}

Figure-\ref{ref:Figure_3.png} shows the probability density function of the MKL distribution for different values of k.

The mean and variance of the MKL distribution must be computed in order to calculate the MSE of the MKL algorithm.

The MKL distribution's mean:

\begin{equation}
\mu_{MKL} =  \int_{-\infty}^{\infty}x.f_{MKL}(x) \,dx\
\end{equation}

The MKL distribution's variance:

\begin{equation}
\sigma_{MKL}^2 =  \left(\int_{-\infty}^{\infty}x^2.f_{MKL}(x) \,dx\right) - \mu_{MKL}^2
\end{equation}

Therefore, The MSE of MKL is calculated as follows.

\begin{equation}
MSE_{MKL} = (\mu_1 - \mu_{MKL})^2 + \sigma_{MKL}^2
\end{equation}

Due to this, the clean samples are assumed to have standard normal distribution and the difference between SGD and MKL errors is plotted for k = 6, n = 10, with noisy samples' averages higher. MKL's MSE is below SGD in the green area. Figure-\ref{ref:Figure_5.png} shows the case where the MSE of the green area MKL is lower than the SGD. As seen in Figure-\ref{ref:Figure_5.png}, MKL has a lower MSE than SGD if the mean of noisy samples is greater than the mean of clean samples.


\begin{figure}[H]
\begin{center}
\includegraphics[width=\columnwidth]{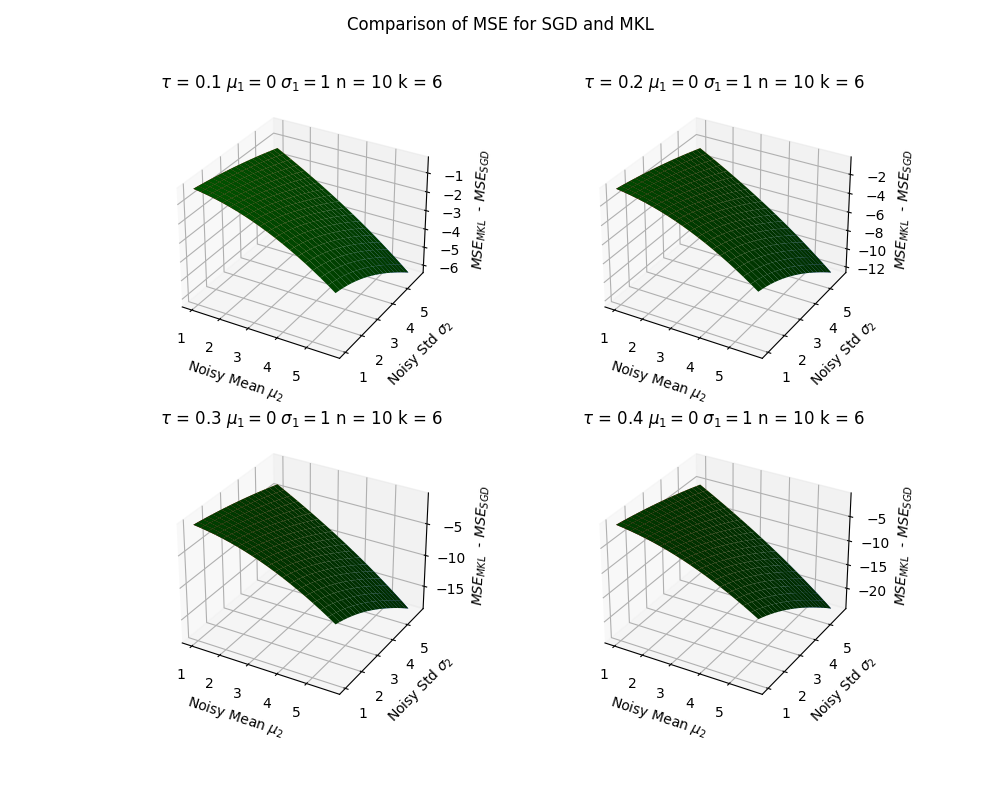}
\caption{Comparison of MSE for SGD and MKL}
\label{ref:Figure_5.png}
\end{center}
\end{figure}

Therefore,
\begin{equation}
    MSE_{MKL} < MSE_{SGD}
\end{equation}

Since all samples are a mixture of clean and noisy samples with low mean loss of clean samples, selecting the smallest k of the selected samples instead of all samples when sampling reduces the mean and variance of the mixture distribution. Consequently, the MKL distribution is more similar to the distribution of clean samples than the mixture distribution, and as a result, the MKL distribution has a lower MSE.

\end{proof}

\begin{theorem}
According to clean sample's distribution, the expected error of applying Adaptive-k algorithm (updating with samples less than a certain threshold value in mini-batch) is smaller than MKL algorithm.
\begin{equation}
    MSE_{adk} < MSE_{MKL}
\end{equation}
\end{theorem}

\begin{proof}

The Adaptive-k algorithm's distribution corresponds to the samples below the mixture distribution's mean. 
Equation-16 defines the pdf of Adaptive-k distribution.

\begin{equation}
f_{adk}(x) = 
\begin{cases}
f_D(x).\frac{1}{F_D(\mu_D)}&\text{if $x<=\mu_D$}\\
0&\text{else}
\end{cases}
\end{equation}

If $x$ is greater than the mean of the mixture distribution, it is 0, if the $x$ is less than or equal to the mean of the mixture distribution, the mixture distribution's normalized pdf is used. 

An example of the Adaptive-k distribution for $\mu_1 = 0$, $\sigma_1 = 1$, $\mu_2 = 5$, $\sigma_2 = 2$, $\tau = 0.4$   can be seen in Figure-\ref{ref:Figure_2.png}.

\begin{figure}[H]
\begin{center}
\includegraphics[width=\columnwidth]{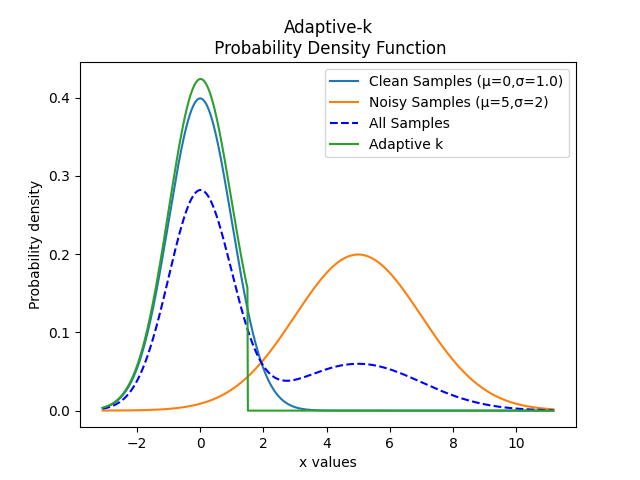}
\caption{Adaptive-k distribution's pdf for noisy dataset where $\mu_1 = 0$, $\sigma_1 = 1$, $\mu_2 = 5$, $\sigma_2 = 2$, $\tau = 0.4$ }
\label{ref:Figure_2.png}
\end{center}
\end{figure}

MSE can be calculated as the sum of the variance and squared bias of the estimator. The Adaptive-k distribution's mean and variance are calculated with the following equations.

The Adaptive-k distribution's mean:

\begin{equation}
\mu_{adk} =  \int_{-\infty}^{\mu_D}x.f_{adk}(x) \,dx\
\end{equation}

The Adaptive-k distribution's variance:

\begin{equation}
\sigma_{adk}^2 =  \left(\int_{-\infty}^{\infty}x^2.f_{adk}(x) \,dx\right) - \mu_{adk}^2
\end{equation}

The MSE of Adaptive-k is calculated as follows.

\begin{equation}
MSE_{adk} = (\mu_1 - \mu_{adk})^2 + \sigma_{adk}^2
\end{equation}

There is no exact analytical solution for the Adaptive-k and MKL distributions because their MSE's contain the erf function. For comparison of the Adaptive-k and MKL MSE's, while keeping the noiseless sample distribution parameters constant, the differences between them were plotted according to the noisy sample's mean $\mu_2$, noisy sample's standard deviation $\sigma_2$, and the noise ratio $\tau$ parameters. For the numerical calculation, the clean samples are considered constants $\mu_1 = 0$ and $\sigma_1 = 1$. Separate calculations were made for noise ratios $\tau = 0.1$, $\tau = 0.2$, $\tau = 0.3$ and $\tau = 0.4$. And it is used $n = 10$ as a constant for MKL, and $k = 6$ as its optimal value.

\begin{figure}[H]
\begin{center}
\includegraphics[width=\columnwidth]{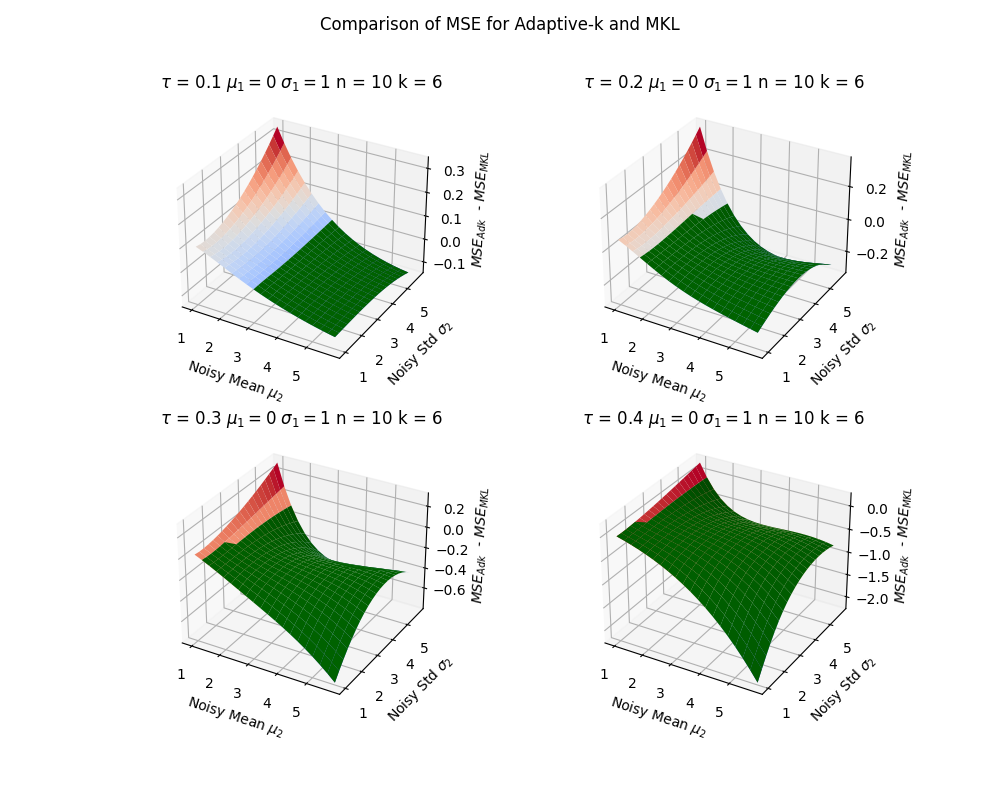}
\caption{Comparison of MSE for Adaptive k and MKL}
\label{ref:Figure_4.png}
\end{center}
\end{figure}

Figure-\ref{ref:Figure_4.png} shows the comparison between Adaptive-k and MKL MSE. Surfaces with Adaptive-k MSE lower than MKL MSE are colored green. As the noise ratio and mean of noisy samples increase, Adaptive-k's error decreases as compared to MKL on surfaces. In the theoretical work, it is assumed that the average loss of noisy samples is higher than that of clean samples. Our experiments and other studies confirm this. Adaptive-k provides consistently better results in areas where Adaptive-k has a lower MSE, as shown in the experiments. 

Therefore,
\begin{equation}
    MSE_{adk} < MSE_{MKL}
\end{equation}

\end{proof}

MKL has a lower MSE in areas with low noise loss mean, low noise ratio, and high noise loss variance. The parameters are associated with the separation of the two distributions. As the losses and noise ratio of the noisy samples decrease and as the variance increases, identifying the distribution of the noisy samples becomes more difficult. However, the Adaptive-k algorithm has a lower MSE when the two distributions are distinguishable.

In this theoretical analysis, the loss distribution of the samples was taken as constant. However, during model training, the loss distribution of the samples changes. Therefore, the mean loss of all samples must be estimated or calculated at each step. Adam is one of the most successful and frequently used methods for estimating gradients. By using samples from each mini-batch, Adaptive-k estimates the average loss of all samples as in the Adam algorithm.

\begin{figure}[H]
\begin{center}
\includegraphics[width=\columnwidth]{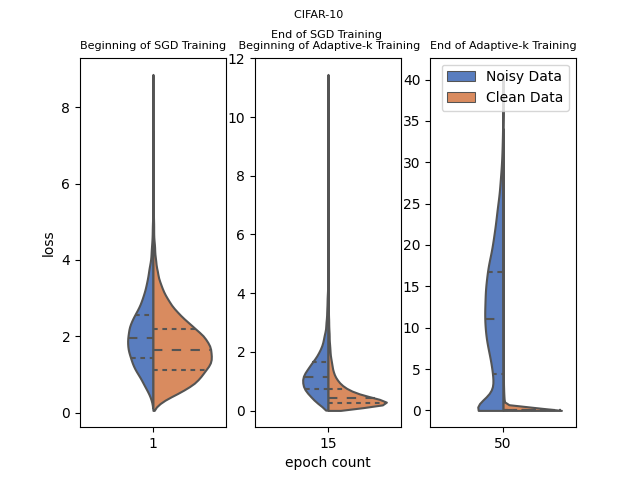}
\caption{Loss Distributions of Clean and Noisy Data Throughout Training for Noise Ratio $\tau = 0.3$}
\label{ref:Figure_7.png}
\end{center}
\end{figure}

Figure-\ref{ref:Figure_7.png} shows the loss distributions of clean and noisy samples throughout the Adaptive-k training. The losses of the samples are randomly distributed according to the initial weights of the model at the start of model training. As a result, at the beginning of SGD training, the distributions of noisy and clean samples are not distinguishable. The mean and variance of the two distributions differ and can be distinguished when the model is trained with the SGD algorithm for a certain number of epochs. After this stage, the model is trained with the Adaptive-k algorithm, whereby noisy samples are omitted from the training and the model performance improves. Thus, we recommend training the model with SGD using all samples before Adaptive-k training. We show in our experiments that training the model with SGD and passing it to MKL or Adaptive-k later on increases the performance of both algorithms.

\section{Discussion and Conclusions}
By detecting label noisy samples and not including them in training, deep neural networks can be made more robust. It is also possible to detect noisy samples if it is not known which samples are labeled noisy by making use of the tendency of noisy samples to have higher loss during model training. However, the loss values of the samples change dynamically during training. As a result, it can be difficult to detect outliers. Our proposed Adaptive-k algorithm uses the samples selected for the mini-batch and follows the mean loss values over the dataset using the Adam algorithm. By doing so, we detect and eliminate noisy samples for each mini-batch, if any. It is possible to make a better distinction using Adaptive-k since it utilizes information from all samples throughout training, not just the current mini-batch. Thus, instead of eliminating a fixed number of samples, it adaptively eliminates some samples for the current mini-batch. The theoretical analysis shows that under certain conditions, for datasets with a certain amount of label noise, the Adaptive-k algorithm has a lower MSE than the MKL and SGD algorithms. When using noisy datasets, the Adaptive-k method performs better than the MKL and SGD algorithms, approaching the Oracle method that does not use noisy samples at all. Thus, the Adaptive-k algorithm increases the robustness of the model for the dataset with label noise.

In our experimental studies, we show the effectiveness of the Adaptive-k algorithm for 3 images and 4 text datasets. It can be said that Adaptive-k consistently performs better than other algorithms, especially for image datasets. Adaptive-k can detect different noise ratios with high precision using the same settings. Adaptive-k is not specific to image datasets; it can be applied to any type of dataset. It does not involve applying to an external source or using any other parallel model training. Therefore, it does not require a longer training time. 

In future studies, we would like to investigate the usage and effect of threshold-based approaches such as Adaptive-k in semi-supervised training.

\section*{Acknowledgment}
This study was supported by the Scientific and Technological Research Council of Turkey (TUBITAK) Grant No: 120E100.

\bibliography{citation}
\bibliographystyle{IEEEtran}
\end{document}